\RequirePackage{amsmath} 
\documentclass[runningheads]{llncs}
\usepackage[T1]{fontenc}
\usepackage{graphicx}

\usepackage{placeins} 
\usepackage{algorithm}
\usepackage[noEnd=true]{algpseudocodex}
\usepackage{amssymb,enumitem,hyperref,booktabs}
\usepackage{mathrsfs}
\usepackage[bb=dsserif]{mathalpha} 
\usepackage{color}

\def\correction#1{{\color{blue}{#1}}}
\usepackage[textsize=tiny,textwidth=4cm,backgroundcolor=yellow,bordercolor=white]{todonotes}
\makeatletter 
\@mparswitchfalse%
\makeatother
\normalmarginpar

\usepackage{float}
\newfloat{algorithm}{t}{lop}

\def\A{{\cal A}}
\def\E{\mathbb{E}}
\def\H{{\cal H}}

\def\M{{\cal M}}
\def\O{{O}}

\def\Q{{\cal Q}}
\def\reals{\mathbb{R}}
\def\R{{\cal R}}
\def\S{{\cal S}}
\def\T{{\cal T}}
\def\V{{\cal V}}
\def\X{{\cal X}}

\DeclareMathOperator{\conv}{conv}

\def\suchthat{\mid} 

\def\Easp{{\cal E}} 

\DeclareMathOperator{\KLdiv}{\mathrm{KLdiv}}
\DeclareMathOperator{\Var}{\mathrm{Var}}

\begin{document}
\title{Non-maximizing 
policies that fulfill multi-criterion aspirations in expectation
}
\titlerunning{Policies that fulfill multi-criterion aspirations in expectation}
%
\author{
Simon Dima\inst{1}\orcidID{0009-0003-7815-8238} \and
Simon Fischer\inst{2}\orcidID{0009-0000-7261-3031} \and
Jobst Heitzig\inst{3}\orcidID{0000-0002-0442-8077} \and
Joss Oliver\inst{4}\orcidID{0009-0008-6333-7598}
}
%
%
\institute{
École Normale Supérieure, Paris, France, \email{simon.dima@ens.psl.eu} \and 
Independent Researcher, Cologne, Germany
\and
PIK, Potsdam, Germany, \email{heitzig@pik-potsdam.de}, Corr. Author \and
Independent Researcher, London, UK, \email{joss.oliver62@gmail.com}\\[2mm]
\correction{Slightly corrected version (changes in blue), 25 February 2025}
}
\maketitle              
\begin{abstract} 
In dynamic programming and reinforcement learning,
the policy for the sequential decision making of an agent in a stochastic environment is usually determined by expressing the goal as a scalar reward function and seeking a policy that maximizes the expected total reward.
However, many goals that humans care about naturally concern multiple aspects of the world, and it may not be obvious how to condense those into a single reward function.
Furthermore, maximization suffers from specification gaming,
where the obtained policy achieves a high expected total reward in an unintended way,
often taking extreme or nonsensical actions.

Here we consider finite acyclic Markov Decision Processes with multiple distinct evaluation metrics, which do not necessarily represent quantities that the user wants to be maximized.
We assume the task of the agent is to ensure that the vector of expected totals of the evaluation metrics falls into some given convex set,
called the aspiration set.
Our algorithm guarantees that this task is fulfilled
by using simplices to approximate feasibility sets and propagate aspirations forward while ensuring they remain feasible.
It has complexity linear in the number of possible state--action--successor triples and polynomial in the number of evaluation metrics.
Moreover, the explicitly non-maximizing nature of the chosen policy and goals yields additional degrees of freedom,
which can be used to apply heuristic safety criteria to the choice of actions.
We discuss several such safety criteria
that aim to steer the agent towards more conservative behavior.
\end{abstract}
%
%
%
\section{Introduction}
In typical reinforcement learning (RL) and dynamic programming problems an agent is trained or programmed to solve tasks encoded by a single real-valued reward function that it shall maximize. However, many tasks are not easily expressed by such a function \cite{subramani2023expressivity}, human preferences are hard to learn and may not be easy to aggregate across stakeholders \cite{conitzer2024social}, and maximizing a misspecified objective may fall prey to reward hacking \cite{amodei2016concrete} and Goodhart's law \cite{skalse2023invariance}, leading to unintended side-effects and potentially harmful consequences.

In this work, we study a particular \textit{aspiration-based} approach to agent design. We assume an existing task-specific world model in the form of a fully observed Markov Decision Process (MDP), where the task is not encoded by a reward function but instead a \textit{multi-criterion evaluation function} and a bounded, convex subset of its range, called an \textit{aspiration set}, that can be thought of as an ``instruction'' \cite{clymer2023generalization} to the agent.
Aspiration-type goals can also naturally arise from subtasks in complex environments even if the overall goal is to maximize some objective, when the complexity requires a hierarchical decision-making approach whose highest level selects subtasks that turn into aspiration sets for lower hierarchal levels.

In our version of aspiration-based agents, the goal is to make the \emph{expected value} of the total with respect to this evaluation function fall within the aspiration set, and select from this set according to certain performance and safety criteria. 
They do so step-wise, exploiting recursion equations similar to the Bellman equation.
Thus our approach is like multi-objective reinforcement learning (MORL), with a primary aspiration-based objective and at least one secondary objective incorporated via action-selection criteria \cite{VAMPLEW2021104186}. Unlike MORL, the components of the evaluation function (called \textit{evaluation metrics}) are not objectives in the sense of targets for maximization. Rather, an aspiration formulated w.r.t.\ several evaluation metrics might correspond to a single objective (e.g., ``make a cup of tea''). Also, at no point does an aspiration-based agent aggregate the evaluation metrics into a single value. Instead, any trade-offs are built into the aspiration set itself, similar to what \cite{dalrymple2024towards} call a ``safety specification''. For example, aspiring to buy a total of 10 oranges and/or apples for at most EUR 1 per item could be encoded with the aspiration set $\{(o, a, c) \::\: o,a \geq 0; c\le o + a = 10\}$.

A similar set-up to ours has been used in \cite{Miryoosefi2019} which has used reinforcement learning to find a policy whose expected discounted reward vector lies inside a convex set. Instead of a reinforcement \emph{learning} perspective, we use a model-based \emph{planning} perspective and design an algorithm that explicitly calculates a policy for solving the task, based on a model of the environment.
\footnote{Nevertheless, our algorithm can also be straightforwardly adapted to learning.}
Also, the approach in \cite{Miryoosefi2019} is concerned with guaranteed bounds for the distance between received rewards and the convex constraints in terms of the number of iterations, whereas we focus on guaranteeing aspiration satisfaction in a fixed number of computational steps, providing a verifiable guarantee in the sense of \cite{dalrymple2024towards}.

Other agent designs that follow a non-maximization-goal-based approach include quantilizers, decision transformers and active inference. Quantilizers are agents that use random actions in the top $n\%$ of a ``base distribution'' over actions, sorted by expected return \cite{taylor2015quantilizers}. The goal for decision transformers is to make the expected return equal a particular value $R_{target}$ \cite{chen2021decision}. The goal for active inference agents is to produce a particular probability distribution of observations \cite{tschantz2020reinforcement}. While the goal space for quantilizers and decision transformers, being based on a single real-valued function,
is often too restricted for many applications, that of active inference agents (all probability distributions) appears too wide for the formal study of many aspects of aspiration-based decision-making. Our approach is of intermediary complexity.

An important consideration in this work, to ensure tractability in large environments, is also the \emph{computational} complexity in the number of actions, states, and evaluation metrics. 
We will see that for our algorithm, the preparation of an episode has linear complexity in the number of possible state--action--successor transitions and (conjectured and numerically confirmed) linear average complexity in the number of evaluation metrics, and then the additional per-time-step complexity of the policy is linear in the number of actions, constant in the number of states, and polynomial in the number of evaluation metrics.

Our work also affects the emerging AI safety/alignment field, which views unintended consequences from maximization, e.g., reward hacking and Goodhart's law, as a major source of risk once agentic AI systems become very capable \cite{amodei2016concrete}.

\section{Preliminaries}

\paragraph{Environment.}
An \emph{environment} $E=(\S,s_0,\S_\top,\A,\T)$ is a finite Markov Decision Process without a reward function,
consisting of
a finite \emph{state space} $\S$, 
an \emph{initial state} $s_0\in\S$, 
a nonempty subset $\S_\top\subseteq\S$ of \emph{terminal states}, 
a nonempty finite \emph{action space} $\A$, 
and a function $\T:(\S\setminus\S_\top)\times\A\to\Delta(\S\setminus\{s_0\})$
specifying \emph{transition probabilities}:
$\T(s, a)(s')$ is the probability that taking action $a$ from state $s$ leads to state $s'$.
We assume that the environment is acyclic, i.e.,
that it is impossible to reach a given state again after leaving it.
We fix some environment $E$ and write $s'\sim s,a$ to denote that $s'$ is distributed according to $\T(s,a)$.


\paragraph{Policy.}
A \emph{(memory-based) policy} is given by
some nonempty finite set $\M$ of memory states internal to the agent,
an initial memory state $m_0 \in \M$ and
a function $\pi:\M\times(\S\setminus\S_\top)\to\Delta(\A\times\M)$
that maps each possible combination of memory state $m\in\M$
and (environment) state $s\in\S\setminus\S_\top$
to a probability distribution over combinations of actions $a\in\A$ and successor memory states $m'\in\M$.
Let $\Pi_\M$ be the set of all policies with memory space $\M$.
The special class of \emph{Markovian} or memoryless policies is obtained when $\M$ is a singleton.
Policies which are both Markovian and deterministic are called
\emph{pure Markov policies}, and amount to a function $(\S\setminus\S_\top) \rightarrow \A$.
We denote by $\Pi^0$ the set of Markovian policies
and by $\Pi^p$ the set of all pure Markov policies.
%

\paragraph{Evaluation, Delta, Total.}
A \emph{(multi-criterion) evaluation function} for the environment is a function $f:(\S\setminus\S_\top)\times\A\times(\S\setminus\{s_0\})\to\reals^d$ where $d\ge 1$. 
The quantity $f(s,a,s')$ is called the \emph{Delta} received under transition $(s,a)\to s'$. It represents by how much certain evaluation metrics change when the agent takes action $a$ in state $s$ and the successor state is $s'$. 
Let us fix $f$ for the rest of the paper.
The \emph{(received) Total} of a trajectory $h = (m_0, s_0, a_1, m_1, s_1,\dots,a_T, m_T, s_T)$ is then the cumulative Delta received along the trajectory,
\begin{align}
    \tau(h) &= \textstyle\sum_{t=1}^T f(s_{t-1},a_t,s_t). \label{tauhT}
\end{align}

\paragraph{Value functions.}
Given a policy $\pi$ and an evaluation function $f$,
the state value function $V^\pi:\M\times\S\to\reals^d$
is defined as the expected Total accumulated in future steps
while following policy $\pi$.
In particular, $V^\pi(m_0,s_0)$ is the expected Total of the whole trajectory:
$V^\pi(m_0,s_0) = \E(\tau)$.
Likewise, we define the action value function
$Q^\pi:\S\times\A\times\M\to\reals^d$.
These satisfy the Bellman equations
\begin{align}
    V^\pi(m,s) &= \E_{a,m'\sim\pi(m,s)} \left(Q^\pi(s, a, m')\right),
    \label{Vpibellman}
    \\
    Q^\pi(s, a, m) &= \E_{s' \sim s, a} \left(
    f(s,a,s') + V^\pi(m,s') \right).\label{Qpibellman}
\end{align}
They can be calculated by backwards induction since the environment is finite and acyclic, with the base case $V^\pi(m, s) = 0$ for terminal states $s \in \S_\top$.
For memoryless policies, we will elide the argument $m$ since $\M$ is a singleton.

The Delta and Total are analogous to reward and return in an MDP with a reward function,
and the value functions $V$ and $Q$ are defined as usual,
albeit with vector instead of scalar arithmetic.
However, while we do assume that the evaluation metrics represent some aspects relevant to the task at hand,
we do \emph{not} assume that they represent a form of utility
which it is always desirable to increase.
Accordingly, the agent's goal will be specified by the user not as a maximization task,
but rather as a set of linear constraints on the expected sums of the evaluation metrics, which we call an \emph{aspiration}.

\section{Fulfilling aspirations}
\FloatBarrier
\paragraph{Aspirations, feasibility.}
An \emph{(initial) aspiration} is a convex polytope $\Easp_0\subset\reals^d$,
representing the values of the expected total $\E\tau$ which are considered acceptable.
We say that a policy $\pi$ \emph{fulfills} the aspiration when it satisfies $V^\pi(m_0, s_0) \in \Easp_0$.
To answer the question of whether it is possible to fulfill a given aspiration,
we introduce \emph{feasibility sets}.
The \emph{state-feasibility set} of $s \in S$ is the set of possible values for the expected future Total from $s$, under any memory-based policy:
$\V(s) = \{V^\pi(m, s) \mid \M ~\text{finite set}; m_0,m \in \M; \pi \in \Pi_\M\}$;
likewise, we define the \emph{action-feasibility set} $\Q(s, a)$.
It is straightforward to verify that $\V(s\in\S_\top) = \{0\}$,
and that the following recursive equations hold for $s \notin \S_\top$:
\begin{align}
    \Q(s,a) &= \textstyle\E_{s'\sim s,a}\big( f(s,a,s') + \V(s') \big)
    = \sum_{s'} \T(s, a)(s'){\cdot}(f(s, a, s') + \V(s')),
    \label{Qsarec}\\
    \V(s) &= \textstyle\bigcup_{p \in \Delta(\A)} \E_{a \sim p}\Q(s, a)
    = \bigcup_{p \in \Delta(\A)} \sum_{a\in\A} p(a)\Q(s, a)
    .\label{Vsrec}
\end{align}
In this, we use set arithmetic: $r\X+r'\X' = \{rx+r'x' \suchthat x\in\X,x'\in \X'\}$
for $r,r'\in\reals$ and $\X,\X'\subset\reals^d$.
It is clear that feasibility sets are convex polytopes.

\paragraph{Aspiration propagation.}
Algorithm scheme \ref{alg:generalaspset} shows a general manner of fulfilling a feasible initial aspiration,
starting from a given state or state-action pair.
\begin{algorithm*}\footnotesize
\caption{General scheme for fulfilling feasible aspiration sets}
\label{alg:generalaspset}
    \begin{algorithmic}[1]
        \Procedure{FulfillStateAspiration}{$s\in\S\setminus\S_\top$, nonempty $\Easp \subseteq \V(s)$}
        \State Find suitable $\Easp_a \subseteq \Q(s, a)$ for all $a\in\A$, and $p \in \Delta(\A)$, s.t.\ $\E_{a \sim p}(\Easp_a) \subseteq \Easp$. \label{line:actionpropset}
        \State Draw action $a$ from distribution $p$ and do \Call{FulfillActionAspiration}{$s, a, \Easp_a$}.
        \EndProcedure

        \Procedure{FulfillActionAspiration}{$s\in\S, a\in\A$, nonempty $\Easp_a\subseteq\Q(s,a)$}
        \State Find suitable $\Easp_{s'} \subseteq \V(s')$ for all $s' \in \S$ s.t.\
        $\E_{s' \sim s, a}(f(s, a, s') + \Easp_{s'}) \subseteq \Easp_a$. \label{line:statepropset}
        \State Execute action $a$ and observe successor state $s'$.
        \State If $s'$ is terminal, stop; else do \Call{FulfillStateAspiration}{$s', \Easp_{s'}$}.
        \EndProcedure
    \end{algorithmic}
\end{algorithm*}
It memorizes and updates aspirations $\Easp$ and $\Easp_a$,
initially equalling $\Easp_0$.
The agent alternates between being in a certain state $s$
with a \emph{state-aspiration} $\Easp \subseteq \V(s)$,
and being in a state $s$ and having chosen, but not yet performed,
an action $a$,
with an \emph{action-aspiration} $\Easp_a \subseteq \Q(s,a$).
Although this algorithm is written as two mutually recursive functions,
it can be formally implemented by a memory-based policy that memorizes the current aspiration $\Easp$ or $\Easp_a$.

The way the aspiration set $\Easp$ is \emph{propagated} between steps
is the key part. 
The two directions of aspiration propagation are slightly different:
in state-aspiration to action-aspiration propagation,
shown on line \ref{line:actionpropset},
the agent may choose the probability distribution ($p$) over actions,
whereas in action-aspiration to state-aspiration propagation,
shown on line \ref{line:statepropset},
the next state is determined by the environment with fixed probabilities ($\T$).

The correctness of algorithm \ref{alg:generalaspset} follows from the requirements of
lines \ref{line:actionpropset} and \ref{line:statepropset};
that these are possible to fulfill is a consequence of equations \eqref{Vsrec} and \eqref{Qsarec}.
Feasibility of aspirations is maintained as an invariant.


To implement this scheme, we have to specify how to perform aspiration propagation.
The procedure used to select action-aspirations $\Easp_a$ and state-aspirations $\Easp_{s'}$
should preferably allow some control over how the size of these sets changes over time.
On one hand, preventing aspiration sets from shrinking too fast
preserves a wider range of acceptable behaviors in later steps%
\footnote{
However, algorithm \ref{alg:generalaspset} remains correct even if the aspiration sets shrink to singletons.
},
but on the other hand, keeping the aspiration sets
somewhat smaller than the feasibility sets
also provides immediate freedom in the choice of the next action,
as detailed in section \ref{sec:whyshrink}.
An additional challenge is posed by the complex shape of the feasibility sets,
which we must handle in a tractable way.

\paragraph{Approximating feasibility sets by simplices.}
\label{sec:simplices}
Let $\conv(X)$ denote the convex hull of any set $X \subseteq \reals^d$.
Given any tuple $\R$ of $d+1$ memoryless \emph{reference policies}
$\pi_1,\dots,\pi_{d+1}\in\Pi^0$,
we define \emph{reference simplices} in evaluation-space as
$\V^\R(s) = \conv\{ V^\pi(s) \suchthat \pi\in\R \}$
and $\Q^\R(s,a) = \conv\{ Q^\pi(s,a) \suchthat \pi\in\R \}$.
It is immediate that these are subsets of the convex feasibility sets
$\V(s)$ resp. $\Q(s,a)$,
and that 
\begin{align}
    \V^\R(s) &\subseteq \textstyle\bigcup_{p \in \Delta(\A)} \E_{a\sim p}\left(\Q^\R(s, a)\right),
    \\
    \Q^\R(s,a) &\subseteq \E_{s'\sim s,a}\big( f(s,a,s') + \V^\R(s') \big).
\end{align}
These imply that we can replace every occurrence of $\V$ and $\Q$ in algorithm \ref{alg:generalaspset}
with $\V^\R$ resp. $\Q^\R$,
obtaining a correct algorithm to guarantee fulfillment of any initial aspiration
$\Easp_0$ provided it intersects the reference simplex $\V^\R(s_0)$.

It turns out that the latter can be guaranteed by a proper choice of reference policies, and that we can always use pure Markov policies for this:
\begin{lemma}
For any state $s$, we have
$\V(s) = \conv \{V^\pi(s) \suchthat \pi \in \Pi^p\}$.
\end{lemma}
\begin{proof}
Any memory-based policy admits a Markovian policy with the same occupancy measure and hence the same expected Total \cite{feinberg1996stationary}.
Hence $\V(s) = \{V^\pi(s) \suchthat \pi \in \Pi^0\}$.
A convex set is the convex hull of its vertices,
so $\V(s) = \conv\{V^\pi(s) \suchthat \exists y \in \reals^d, \pi = \arg\max_{\pi' \in \Pi^0} (y\cdot \V^{\pi'}(s))\}$.
Finally, maximal policies may be taken to be deterministic, which concludes the argument.
\qed
\end{proof}
As a consequence, for any aspiration set $\Easp$ intersecting the feasibility set $\V(s_0)$,
there exists a tuple $\R$ of pure Markov reference policies
such that $\V^\R(s_0) \cap \Easp \neq \emptyset$.
Section \ref{sec:refsimp} describes a heuristic algorithm for finding such reference policies.



We now turn to explaining a way to enact the aspiration-propagation steps needed in lines
\ref{line:actionpropset} and \ref{line:statepropset} of algorithm \ref{alg:generalaspset}, based on shifting and shrinking.

\section{Propagating aspirations}
\subsection{Propagating action-aspirations to state-aspirations}
To implement algorithm schema \ref{alg:generalaspset}, we first focus on line \ref{line:statepropset} in procedure \textsc{FulfillActionAspiration}, which is the easier part. Given a state-action pair $s, a$
and an action-aspiration set $\Easp_a \subseteq \Q^\R(s, a)$,
we must construct nonempty state-aspiration sets $\Easp_{s'} \subseteq \V^\R(s')$
for all possible successor states $s'$,
such that $\E_{s' \sim s, a}(f(s, a, s') + \Easp_{s'}) \subseteq \Easp_a$.
We assume that all reference simplices are nondegenerate,
i.e. have full dimension $d$. 
This is almost surely the case if there are sufficiently many actions and these have enough possible consequences.

\paragraph{Tracing maps.}
Under this assumption,
we define
tracing maps $\rho_{s, a, s'}$ from the reference simplex
$\Q^\R(s,a)$ to $\V^\R(s')$.
Since domain and codomain are simplices,
we can choose $\rho_{s,a,s'}$ to be the unique affine linear map that maps vertices to vertices,
$\rho_{s, a, s'}(Q^{\pi_i}(s, a)) = V^{\pi_i}(s')$. 
For any point $e \in \Q^\R(s, a)$, it follows from equation \eqref{Vpibellman}
that $\E_{s' \sim s,a}(f(s, a, s') +  \rho_{s, a, s'}(e)) = e$.
Accordingly, to propagate aspirations of the form $\Easp_a = \{e\}$,
it is sufficient to just set $\Easp_{s'} = \{\rho_{s, a, s'}(e)\}$.
However, for general subsets $\Easp_a \subseteq \Q^\R(s, a)$,
the set
$\E_{s' \sim s,a}(f(s, a, s') +  \rho_{s, a, s'}(\Easp_a))$
is in general strictly larger than $\Easp_a$,
and hence we map $\Easp_a$ in a different way. 

For this, choose an arbitrary ``anchor'' point $e\in \Easp_a$;
here we let $e$ be the average of the vertices, $C(\Easp_a)$,
but any other standard type of center would also work (e.g. analytic center, center of mass/centroid, Chebyshev center).
Now, let $\X_{s'} = \Easp - e + \rho_{s, a, s'}(e)$ be shifted copies of the action-aspiration.
We would like to use the $\X_{s'}$ as state-aspirations, and indeed they have the property that
\begin{align}
\E_{s' \sim s,a}(f(s, a, s') + \X_{s'}) &= \E_{s' \sim s,a}( f(s, a, s') + \rho_{s, a s'}(e) - e + \Easp_a) \\
&= e - e + \E_{s'\sim s, a} (\Easp_a) = \Easp_a \quad\text{as $\Easp_a$ is convex.}
\end{align}
This is almost what we want, but it might be that $\X_{s'}$ is not a subset of $\V^\R(s')$.
To rectify this,
we opt for a shrinking approach, setting
$\Easp_{s'} = r_{s'}\cdot(\Easp_a - e) + \rho_{s, a, s'}(e)$
for the largest $r_{s'} \in [0, 1]$ such that the result is a subset of $\V^\R(s')$.
As $\Easp_{s'}$ does not depend on any other successor state $s''\neq s'$, we can wait until the true $s'$ is known and only compute $\Easp_{s'}$ for that $s'$, which saves computation time.

\begin{proposition}\label{Ccomplexity}
    Given all values $V^{\pi_i}(s)$ and $Q^{\pi_i}(s,a)$ and both the $k_c\ge d+1$ constraints and $k_v\ge d+1$ vertices defining $\Easp_a$, 
    the shrinking version of action-to-state aspiration propagation has 
    time complexity $O( [k_c^{1.5}d + (dk_v)^{1.5}]L )$,
    where $L$ is the precision parameter defined in \cite{Vaidya1989}.
    If $\Easp_a$ is a simplex, this is $O(d^3 L)$.
\end{proposition}
\begin{proof}
    A linear program (LP) with $m$ constraints and $n$ variables has time complexity $O(f(m,n)L)$, where $f(m,n)=(m+n)^{1.5}n$ and $L$ is a precision parameter \cite{Vaidya1989}.
    $C(\Easp_a)$ can be calculated with an LP with $m=k_c+1$ and $n=d+1$. 
    Finding the convex coefficients for $\rho_{s,a,s'}$ requires solving a system of $d+1$ linear equations,
    needing time $O(d^\omega)$, where $2 \leq \omega < 2.5$ is the exponent of matrix multiplication. 
    Finding the shrinking factor $r$ is an LP with $m=(d+1)k_v$ and $n=1$,
    and then computing the constraints and vertices of $\Easp_{s'}$ is $O(d(k_c+k_v))$. 
    In all, this gives a time complexity of
    $\O( Lf(k_c,d) + d^\omega + Lf(dk_v,1) + d(k_c+k_v) )
    \le \O( L(k_c+d)^{1.5}d + d^\omega + L(dk_v)^{1.5} )
    \le \O( L(k_c^{1.5}d + (dk_v)^{1.5}) ) $.
    \qed
\end{proof}

\subsection{Choosing actions and action-aspirations}
\label{sec:choosingactions}
This is the core of our construction.
In state $s$ with state-aspiration $\Easp$, the policy probabilistically selects an action $a$ and action-aspiration $\Easp_a$ as follows:
\begin{itemize}[leftmargin=10mm]
\item[(i)] For each of the $d+1$ vertices $V^{\pi_i}(s)$ of the state's reference simplex $\V^\R(s)$, find a \emph{directional action set} $\A_i\subseteq\A$ containing those actions whose reference simplices $\Q^\R(s,a)$ ``lie between'' $\Easp$ and the vertex $V^{\pi_i}(s)$.
\item[(ii)] From the full action set $\A_0=\A$ and from each directional action set $\A_i$ ($i=1\dots d+1$) independently, use some arbitrary, potentially probabilistic procedure to select one element, giving \emph{candidate actions} $a_0,\dots,a_{d+1}$.
\item[(iii)] For each candidate $a_i$, compute an action-aspiration $\Easp_{a_i}$ by shifting and shrinking the state-aspiration $\Easp$ into the reference simplex $\Q^\R(s,a_i)$.  
\item[(iv)] Compute a probability distribution $p\in\Delta(\{0,\dots,d+1\})$ that makes\\ $\E_{i\sim p}\Easp_{a_i}\subseteq\Easp$ and has as large a $p_0$ as possible. 
\item[(v)] Finally execute candidate action $a_i$ with probability $p_i$ ($i=0\dots d+1$) and memorize its action-aspiration $\Easp_{a_i}$.
\end{itemize}
We now describe steps (i)--(iv) in detail. 
Algorithm~\ref{algC} has the corresponding pseudocode, 
and Figure~\ref{figC} illustrates the involved geometric operations.

\begin{algorithm*}
\footnotesize
\caption{Action selection (``shrinking'' variant)\label{algC}}
\begin{algorithmic}[1]
    \Require reference simplex vertices $V^{\pi_i}(s)$, $Q^{\pi_i}(s,a)$; state $s$; state-aspiration $\Easp$; shrinking schedule $r_{\max}(s)$.
    \State $x \gets C(\Easp_s)$; $\Easp' \gets \Easp - x$ \Comment{center, shape}
    \For{$i = 0, \dots, d+1$}
        \State $\A_i \gets$ \Call{DirectionalActionSet}{$i$}
        \State $a_i \gets$ \Call{SampleCandidateAction}{$\A_i$} \Comment{any (probabilistic) choice from $\A_i$}
        \State $\Easp_{a_i} \gets$ \Call{ShrinkAspiration}{$a_i,i$} \Comment{action-aspiration}
    \EndFor
    \State solve linear program $p_0=\max$!, $\sum_{i=0}^{d+1} p_i \Easp_{a_i} \subseteq \Easp$ for $p\in\Delta(\{0,\dots,d+1\})$
    \label{line:maxp0}
    \State sample $i\sim p$ and execute $a_i$
    \State memorize $m\gets(s,a_i,\Easp_{a_i})$ \\

    \Procedure{DirectionalActionSet}{$i$}
        \State if $i=0$, \Return $\A$
        \State $\X_i \gets \conv\{x,V^{\pi_i}(s)\}$ \Comment{segment from $m$ to $i$th vertex}
        \State \Return $\{ a\in\A \suchthat \Q^\R(s,a)\cap\X_i\neq\emptyset \}$ \Comment{actions with feasible points on $\X_i$}
    \EndProcedure \vspace{2mm}
    
    \Procedure{ShrinkAspiration}{$a_i,i$}
        \State $v \gets V^{\pi_i}(s)$ if $i>0$, else $C(\Q^\R(s,a_i))$
        \Comment{vertex or center of target} 
        \label{line:calcv}
        \State $y \gets v - x$
        \Comment{shifting direction}
        \label{line:calcy}
        \State $r \gets \max\{r \in [0,r_{\max}(s)] \suchthat \exists \ell\ge 0, \ell y + x + r\Easp'\subseteq\Q^\R(s,a)\}$
        \Comment{size of largest shifted and shrunk copy of $\Easp$ that fits into $\Q^\R(s,a)$}
        \State $\ell \gets \min\{ \ell\ge 0 \suchthat \ell y + x + r\Easp'\subseteq\Q^\R(s,a) \}$
        \Comment{shortest dist. that makes it fit in}
        \State \Return $\ell y + x + r\Easp'$
        \Comment{shift and shrink}
    \EndProcedure
    
\end{algorithmic}
\end{algorithm*}

\begin{figure}
\centering
\includegraphics[width=.8\textwidth, trim=0 550 700 0, clip]{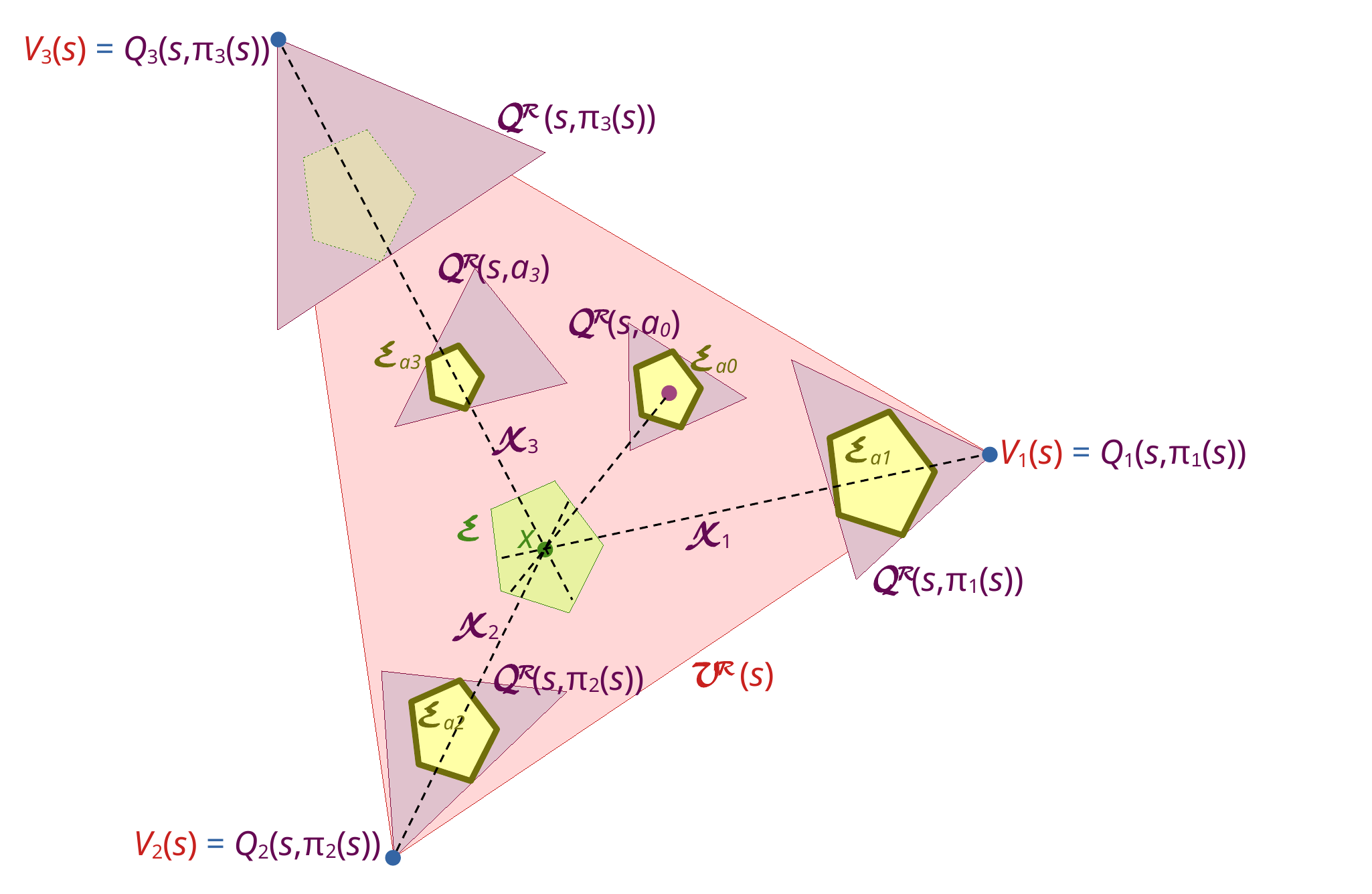}
\caption{Construction of action-aspirations $\Easp_a$ from state-aspiration $\Easp$ and reference simplices $\V^\R(s)$ and $\Q^\R(s,a)$ by shifting and shrinking. See main text for details.} \label{figC}
\end{figure}

\paragraph{(i) Directional action sets.}
Compute the average of the vertices of $\Easp$, $x=C(\Easp)$, and the ``shape'' $\Easp'=\Easp-x$.
For $i=0$, put $\A_i=\A$. 
For $i>0$, let $\X_i$ be the segment from $x$ to the vertex $V^{\pi_i}(s)$ and check for each action $a\in\A$ whether its reference simplex $\Q^\R(s,a)$ intersects $\X_i$, using linear programming. Let $\A_i$ be the set of those $a$, which is nonempty since $\pi_i(s)\in\A_i$ by definition of $V^{\pi_i}(s)$.

\paragraph{(ii) Candidate actions.}
For each $i=0\dots d+1$, use an arbitrary, possibly probabilistic procedure to select a candidate action $a_i\in\A_i$.
Since $\A_0=\A$, any possible action might be among the candidate actions. 
This freedom can be used to improve the policy in terms of the evaluation metrics, e.g., by choosing actions that are expected to lead to a rather low variance of the received Total.
It can also be used to incorporate additional action-selection criteria that are unrelated to these evaluation metrics, to increase overall safety, e.g., by preferring actions that avoid unnecessary side-effects.
We'll discuss this in Section~\ref{sec:criteria}.

\paragraph{(iii) Action-aspirations.}
Given direction $i$ and action $a_i$, we now aim to select a large subset $\Easp_{a_i}\subseteq\Q^\R(s,a_i)$ that fits into a shifted version $z_i + \Easp'$.
We determine a direction $y$ towards which to shift $\Easp_s$: if $i=0$, towards the average of the vertices, $C(\Q^\R(s,a_i))$, otherwise towards the reference vertex $V^{\pi_i}(s)$. This is lines \ref{line:calcv} and \ref{line:calcy} of Alg.~\ref{algC}.
As before, we use shrinking:
find the largest shrinking factor $r \in [0, r_{\max}(s)]$
for which there is a shifting distance $\ell\ge 0$
so that $\ell y + r\Easp'\subseteq\Q^\R(s,a_i)$,
using a linear program with two variables $r,\ell$,
and then find the smallest such $\ell$ for that $r$ using another linear program.
The ``shrinking schedule'' $r_{\max}(s)\in[0,1]$ might be used to enforce some amount of shrinking to increase the freedom for choosing the action mixture, which is the next step.

\paragraph{(iv) Suitable mixture of candidate actions.}
We next find probabilities $p_0,\dots,p_{d+1}$ for the candidate actions so that the corresponding mixture of aspirations $\Easp_{a_i}$ is a subset of $\Easp_s$, $\sum_{i=0}^{d+1}p_i \Easp_{a_i}\subseteq\Easp_s$. 
We show below that this equation has a solution.
\label{sec:whyshrink}
Because we want the action $a_0$ that was chosen freely from the whole action set $\A$ to be as likely as possible, we maximize its probability $p_0$. 
This is done in line \ref{line:maxp0} of Algorithm \ref{algC} using linear programming.
Note that the smaller the sets $\Easp_a$, the looser the set inclusion constraint and thus the larger $p_0$. 
We can influence the latter via the shrinking schedule $r_{\max}(s)$, for example by putting $r_{\max}(s)=(1-1/T(s))^{1/d}$ where $T(s)$ is the remaining number of time steps at state $s$, which would reduce the amount of freedom (in terms of the volume of the aspiration set) linearly and end in a point aspiration for terminal states.
\begin{lemma}\label{lemmasolution}
    The linear program in line 6 of Algorithm~\ref{algC} has a solution.
\end{lemma}
\begin{proof}
Because $x\in\V^\R(s)$, there are convex coefficients $p'$ with $\sum_{i=1}^{d+1}p'_i (V^{\pi_i}(s)-x)=0$. 
As each shifting vector $z_i$ is a positive multiple of $V^{\pi_i}(s)-x$, there are also convex coefficients $p$ with $\sum_{i=1}^{d+1}p_i z_{a_i}=0$.
Since $\Easp_{a_i}\subseteq\Easp_s+z_{a_i}$ and $\Easp_s$ is convex, we then have
\begin{align}\textstyle
    \sum_i p_i \Easp_{a_i} \subseteq \sum_i p_i (\Easp_s + z_{a_i}) \subseteq \Easp_s + \sum_i p_i z_{a_i} = \Easp_s.\label{EEa}\qquad    \qed
\end{align}
\end{proof}

\begin{proposition}
    Given all values $V^{\pi_i}(s)$ and $Q^{\pi_i}(s,a)$ and both the $k_c\ge d+1$ many constraints and $k_v\ge d+1$ many vertices defining $\Easp$, 
    this part of the construction (Algorithm~\ref{algC}) has 
    time complexity $O([k_v^{1.5}d^{2.5}|\A| + (k_vk_c)^{1.5}d]L)$.
    If $\Easp$ is a simplex, this is $O(d^4|\A|L)$.
\end{proposition}
\begin{proof}
    Using the notation from 
    Prop.~\ref{Ccomplexity},
    computing $x$ is $O(f(k_v,d)L)$.
    For each $i,a$, verifying whether $a\in\A_i$ and computing 
    $r,\ell$ is done by LPs with $m\le O(dk_v)$ constraints and $n\le 2$ variables,
    giving $O(d|\A| f(dk_v,2)L) = O(d^{2.5}k_v^{1.5}|\A|L)$. 
    The LP for calculating $p$ has $m=d+4+k_vk_c$ and $n=d+2$,
    hence complexity $O( (k_vk_c)^{1.5}dL )$.
    The other arithmetic operations are of lower complexity.    \qed   
\end{proof}

\FloatBarrier
\section{Determining appropriate reference policies}
\label{sec:refsimp}
First, find some feasible aspiration point $x\in\Easp_0\cap\V(s_0)$ (e.g., using binary search).
We now aim to find policies whose values are likely to contain $x$ in their convex hull, by using backwards induction and greedily minimizing the angle between the vector $\E\tau - x$ and a suitable direction in evaluation space.

More precisely: Pick an arbitrary direction (unit vector) $y_1\in\R^d$, e.g., uniformly at random. Then, for $k=1,2,\dots$ until stopping:
\begin{enumerate}
\item Let $\pi_k$ be the pure Markovian policy $\pi$ defined by backwards induction as 
\begin{align}
    \pi(s) &= \underset{a \in\A}{\mathrm{argmax}}\frac{y_k^\top\correction{z(s,a)}}{||\correction{z(s,a)}||_2}, 
    & \correction{z(s,a)} &\correction{{} = Q^{\pi}(s,a) - \frac{\ell(s)}{\rho(s)+\ell(s)}x},\label{refpol}
\end{align}
\correction{where $\rho(s)\ge 0$ is the minimal no.\ of time steps from $s_0$ to $s$ and $\ell(s)\ge 1$ is the maximal no.\ of time steps after $s$.}
Let $v_k = V^{\pi_k}(s_0)$ be the resulting candidate vertex for our reference simplex. If $k\ge d+1$, run a primal-sparse linear program solver \cite{yen2015sparse} to determine whether $x\in\conv\{v_1,\dots,v_k\}$. If so, the solver will return a basic feasible solution, i.e. $d+1$ many vertices that contain $x$ in their convex hull. Let the policies corresponding to the vertices that are part of the basic feasible solution be our reference policies and stop. Otherwise, continue to step 2 below:
\item Let $e_k = (x - v_k) / ||x - v_k ||_2$ be the unit vector in the direction from $v_k$ to $x$.
\item Let $y_{k+1} = \sum_{i=1}^k e_i / k$ be the average of all those directions. Assuming that $x \in \V(s_0)$ and because of the hyperplane separation theorem, choosing directions like this ensures that the algorithm doesn't loop with $v_{l+1} = v_{l}$ for some $l$. Also, note that to check $x \in \conv\{v_1,\dots,v_k,v_{k+1}\}$ it is sufficient to check $v_{k+1} - x \in \mathrm{cone}\{x-v_1,\cdots,x-v_k\}$, the cone generated by the negative of the vertices centred at $x$. So hunting for a policy whose value lies approximately in the direction $y_{k+1}$ from $x$ gives us a good chance of finding a vertex in the aforementioned cone.
\end{enumerate}
We were not able to prove any complexity bounds for this part, but performed numerical experiments with random binary tree-shaped environments of various time horizons, with only two actions (making the directions $v_k-x$ deviate rather much from $y$ and thus presenting a kind of worst-case scenario) and two possible successor states per step, and uniformly random transition probabilities and Deltas $\in[0,1]^d$. 
These suggest that the expected number of iterations of the above is $O(d)$, which we thus conjecture to be also the case for other sufficiently well-behaved random environments. 
Indeed, even if the policies $\pi_k$ were chosen uniformly \emph{at random} (rather than targeted like here) and the corresponding points $V^{\pi_k}(s_0)$ were distributed uniformly in all directions around $x$ (which is a plausible uninformative prior assumption), then one can show easily (using \cite{wendel1962problem}) that the expected number of iterations would be exactly $2d+1$.\footnote{%
    According to \cite{wendel1962problem}, the probability that we need exactly $k\ge d+1$ iterations is $f(k-1) - f(k)$ with $f(k) = 2^{1-k} \sum_{\ell=0}^{d-1} \binom{k-1}{\ell}$.
    Hence $\E k = \sum_{k=d+1}^{\infty}k(f(k-1) - f(k)) = (d+1) f(d) + \sum_{k=d+1}^\infty f(k)$. It is then reasonably simple to prove by induction that $\sum_{k=d+1}^\infty f(k) = d$, and hence $\E k = 2d+1$.
    }

\section{Selection of candidate actions}
\label{sec:criteria}

As we have seen in section \ref{sec:choosingactions} (ii), when choosing actions, we still have many remaining degrees of freedom. Thus, we can use additional criteria to choose actions while still fulfilling the aspirations. We discuss a few candidate criteria here which are related either to gaining information, improving performance, or reducing potential safety-related impacts of implementing the policy.

For many of the criteria, there are \emph{myopic} versions, which only rely on quantities that are already available at each step in the algorithms presented so far, or \emph{farsighted} versions which depend on the continuation policy and thus have to be specifically computed recursively via Bellman-style equations.

\paragraph{Information-related criteria.} If the used world model is imperfect, one might want the agent to aim to gain knowledge by exploration, e.g. by considering some measure of expected information gain such as the evidence lower bound.  


\subsection{Performance-related criteria}
For now, the task of the agent in this paper has been given by specifying aspiration sets for the expected total of the evaluation function. It is natural to consider extensions of this approach to further properties of the trajectory distribution, e.g. by specifying that the variance of the total should be small.

A simple, myopic approach to reducing variance is preferring actions and action-aspirations that are somehow close to the state aspiration $\Easp$, e.g. by choosing action-aspirations where the Hausdorff distance $\displaystyle d_{\mathrm {H}}(\Easp_a, \Easp)$ is small.
A more principled, farsighted approach would be choosing actions and action-aspirations such that the variance of the resulting total is small. Based on equation \eqref{Vpibellman}, the variance can be computed from the total raw second moment $M^\pi$ as
\begin{align}
     M^\pi(s,a,\Easp_a) &= \E_{s',\Easp_{s'}\sim s,a,\Easp_a}\big[||f(s,a,s')||_2^2 + 2f(s,a,s')\cdot V^\pi(s',\Easp_{s'})\nonumber \\
     &\qquad\qquad\qquad\qquad\qquad + \E_{a',\Easp_{a'}\sim \pi(s',\Easp_{s'})}M^\pi(s',a',\Easp_{a'})\big], \\
     \Var(s,a,\Easp_a) &= M^\pi(s,a,\Easp_a) - ||Q^\pi(s,a,\Easp_a)||_2^2.
\end{align}
Note that computing this farsighted metric requires knowing the continuation policy $\pi$, for which algorithm \ref{algC} does not suffice in its current form as it only samples actions. It is however easy to convert it to an algorithm for computing the whole local policy $\pi(s,\Easp)$, which is described in the Supplement.




\subsection{Safety-related criteria}
As mentioned in the introduction, unintended consequences of optimization can be a source of safety problems, thus we suggest to not use any of the criteria introduced in this section as maximization/minimization goals to completely determine the chosen actions; instead, they can be combined into a loss for a softmin action selection policy $\pi_i(a\in\A_i)\propto\exp\big(-\beta\sum_j\alpha_jg_j(a)\big)$, where $g_j(a)$ are the individual criteria. Indeed, in analogy to quantilizers, choosing among adequate actions at random can by itself be considered a useful safety measure, as a random action is very unlikely to be special in a harmful way.
\paragraph{Disordering potential.}
Our first safety criterion is related to the idea of ``fail safety'', and (somewhat more loosely) to ``power seeking''. More precisely, it aims to prevent the agent from moving the system into a state from which it could make the environment's state trajectory become very unpredictable (bring the environment into ``disorder'') because of an internal failure, or if it wanted to. We define the \emph{disordering potential} at a state to be the Shannon entropy $H^\pi(s_t)$ of the stochastic state trajectory $S_{>t} = (S_{t+1},S_{t+2},\dots)$ that would arise from the policy $\pi$ which maximizes that entropy:
\begin{align}
    H^\pi(s_t) := \E_{s_{>t}|s_t,\pi}(-\log\Pr(s_{>t}|s_t,\pi)).
\end{align}
It is straightforward to compute this quantity using the Bellman-type equations
\begin{align}
    H^\pi(s) &= 1_{s\notin\S_\top}\,\E_{a\sim\pi(s)}(-\log\pi(s)(a) + H^\pi(s,a)), \\
    H^\pi(s,a) &= \E_{s'\sim s,a}(-\log\T(s, a)(s') + H^\pi(s')).
\end{align}
To find the maximally disordering policy $\pi$, we assume $\pi(s')$ and thus $H^\pi(s')$ is already known for all potential successors $s'$ of $s$. Then $H^\pi(s,a)$ is also known for all $a$ and to find $p_a = \pi(s)(a)$ we need to maximize $f(p) = \sum_a p_a(\H^\pi(s,a)-\log p_a)$ such that $\sum_a p_a = 1$. Using Lagrange multipliers, we find that for all $a$, $\partial_{p_a}f(p) = \H^\pi(s,a)-\log p_a - 1 = \lambda$ for some constant $\lambda$, hence $p_a \propto \exp(H^\pi(s,a))$ is a softmax policy w.r.t. future expected Shannon entropy. Therefore
\begin{align}
    \pi(s)(a) &= p_a = \exp(H^\pi(s,a)) / Z, &
    Z &= \textstyle\sum_a\exp(H^\pi(s,a)), \\
    H^\pi(s) &= \log Z = \log \textstyle\sum_a\exp(H^\pi(s,a)).
\end{align}

\paragraph{Deviation from default policy.}
If we have access to a default policy $\pi^0$ (e.g. a policy that was learned by observing humans or other agents performing similar tasks), we might want to choose actions in a way that is similar to this default policy. An easy way to measure this is by using the Kullback--Leibler divergence from the default policy $\pi^0$ to the agent's policy $\pi$. Given that we do not know the local policy $\pi(s)$ yet when we decide how to choose the action in the state $s$, we use an estimate $\hat p_a$ (e.g. $\hat p_a=1/(2+d)$) instead to compute the expected total Kullback--Leibler divergence like
\begin{align}
    \KLdiv(s,\hat p_a,a,\Easp_a) &= 
        \log(\hat p_a / \pi^0(s)(a)) + \E_{s'\sim s,a}\KLdiv(s,g(s, a, \Easp_a, s')), \\
    \KLdiv(s,\Easp) &= 1_{s\notin\S_\top}\,\E_{(a,\Easp_a)\sim\pi(s,\Easp)}\KLdiv(s,\pi(s,\Easp)(a),a,\Easp_a),
\end{align}
where $g:(s,a,\Easp_a,s')\mapsto\Easp_{s'}$ implements action-to-state aspiration propagation.

\section{Discussion and conclusion}

\subsection{Special cases}

\paragraph{A single evaluation metric.}
It is natural to ask what our algorithm reduces to in the single-criterion case $d=1$. 
The reference simplices can then simply be taken to be the intervals $\V^\R(s) = [V^{\pi_{\min}}(s), V^{\pi_{\max}}(s)]$ and $\Q^\R(s,a) = [Q^{\pi_{\min}}(s,a), Q^{\pi_{\max}}(s,a)]$, where $\pi_{\min},\pi_{\max}$ are the minimizing and maximizing policies for the single evaluation metric.
Aspiration sets are also intervals, and action-aspirations $\Easp_a$ are constructed by shifting the state-aspiration $\Easp$ upwards or downwards into $\Q^\R(s,a)$ and shrinking it to that interval if necessary.
To maximize $p_{a_0}$, the linear program for $p$ will assign zero probability to that ``directional'' action $a_1$ or $a_2$ whose $\Easp_a$ lies in the same direction from $\Easp$ as $\Easp_{a_0}$ does. In other words, the agent will mix between the ``freely'' chosen action $a_0$ and a suitable amount of a ``counteracting'' action $a_1$ or $a_2$ in the other direction.

\paragraph{Relationship to satisficing.}
A subcase of the $d=1$ case is when the upper bound of the initial state-aspiration interval coincides with the maximal possible value, $\Easp_0 = [e, V^{\pi_{\max}}(s_0)]$, i.e., when the goal is to achieve an expected Total of at least $e$. 
The agent then starts out as a form of ``satisficer'' \cite{simon1956rational}. 
However, due to the shrinking of aspirations over time, aspiration sets of later states $s'$ might no longer be of the same form but might end at values strictly lower than $V^{\pi_{\max}}(s')$ if the interval $[V^{\pi_{\min}}(s'), V^{\pi_{\max}}(s')]$ is wider than the interval $[Q^{\pi_{\min}}(s,a), Q^{\pi_{\max}}(s,a)]$. 
In other words, even an initial satisficer can turn into a ``proper'' aspiration-based agent in our algorithm that avoids maximization in more situations than a satisficer would. 
In particular, also the form of satisficing known as ``quantilization'' \cite{taylor2015quantilizers}, where all feasible expected Totals above some threshold get positive probability, is not a special case of our algorithm.
One can however change the algorithm to quantilization behaviour by constructing successor state aspirations differently, by simply applying the tracing map to the interval, $\Easp_{s'} = \rho_{s,a,s'}[\Easp_a]$ (which is not feasible for $d>1$).

\paragraph{Probabilities of desired or undesired events.}
Another special case is when $d>1$ but the $d$ evaluation metrics are simply indicator functions for certain events. E.g., assume all Deltas are zero except when reaching a terminal state $s'\in\S_\top$, in which case $f_i(s,a,s') = \mathbb 1(s'\in E_i)$ for some subset of desirable or undesirable states $E_i\subseteq\S_\top$.
If the first $k\le d$ many events are desirable in the sense that we want each probability $\Pr(E_i)$ to be $\ge\alpha$ for some $\alpha<1$, and the other $d-k$ many events are undesirable in the sense that we want each probability $\Pr(E_j)$ to be $\le\beta$ for some $\beta>0$, then we can encode this goal as the initial aspiration set $\Easp_0 = [\alpha,1]^k\times[0,\beta]^{d-k}$. 
Note that the different events need not be independent or mutually exclusive, as long as the aspiration is feasible.
Aspirations of this type might be especially natural in combination with methods of inductive reasoning and belief revision that are also based on this type of encoding \cite{kern2024inductive}. This could eventually be useful for a ``provably safe'' approach to AI \cite{dalrymple2024towards}. 

\subsection{Relationship to reinforcement learning}
Even though we formulated our approach in a planning framework where the environment's transition probabilities are known and simple enough to admit dynamic programming, it is clear from Eq.~\eqref{refpol} that the required reference policies $\pi$ and corresponding reference vertices $V^\pi(s)$, $Q^\pi(s,a)$ can in principle also be approximated by reinforcement learning techniques such as (deep) expected SARSA in more complex environments or environments that are given only as samplers without access to transition probabilities. 
For the single-criterion case, preliminary results from numerical experiments suggest that this is indeed a viable approach.\footnote{Note however that some safety-related action selection criteria, especially those based on information-theoretic concepts, require access to transition probabilities which would then have to be learned in addition to the reference simplices.} Future work should explore this further and also consider using approximate dynamic programming methods (e.g., \cite{bonet2009solving}).

If the expected number of learning passes needed to find the necessary reference policies is indeed $O(d)$ as conjectured (see end of Section \ref{sec:refsimp})\footnote{Farsighted action selection criteria would require an additional learning pass to also learn the actual policy and the resulting action evaluations.}, our approach might turn out to have much lower average complexity than the alternative reinforcement learning approach to convex aspirations from \cite{Miryoosefi2019}, which appears to require up to $O(\epsilon^{-2})$ many learning passes to achieve an error of less than $\epsilon$.

\subsection{Invariance under reparameterization}
For many applications there will be several possible parameterizations of the $d$-dimensional evaluation space into $d$ different evaluation metrics, so the question arises which parts of our approach are invariant under which types of reparameterizations of evaluation space.
It is easy to see that all parts are invariant under affine transformations, except for the algorithm for finding reference policies which is only invariant under orthogonal transformations since it makes use of angles, and except for certain safety criteria such as total variance.




\begin{credits}
\subsubsection{\ackname} We thank the members of the \href{https://pik-gane.github.io/satisfia/}{SatisfIA project}, \href{https://aisafety.camp/}{AI Safety Camp}, the \href{https://sparai.org/}{Supervised Program for Alignment Research}, and the organizers of the \href{https://vaisu.ai/}{Virtual AI Safety Unconference}.

\subsubsection{Supplementary Materials.}
This article is accompanied by a supplementary text, containing alternative versions of the main algorithm, and a supplementary video illustrating the evolution of action-aspirations over a sample episode with $d=2$. The text, which contains a link to the video, is available at \url{https://doi.org/10.5281/zenodo.13318993}, and there is Python code available at \url{https://doi.org/10.5281/zenodo.13221511}.

\subsubsection{\discintname}
The authors have no competing interests to declare. 

\subsubsection{CRediT author statement.}
Authors are listed in alphabetical ordering and have contributed equally.
Simon Dima: Formal analysis, Writing - Original Draft, Writing - Review \& Editing.
Simon Fischer: Formal analysis, Writing - Original Draft, Writing - Review \& Editing.
Jobst Heitzig: Conceptualization, Methodology, Software,  Writing - Original Draft, Writing - Review \& Editing, Supervision.
Joss Oliver: Formal analysis, Writing - Original Draft, Writing - Review \& Editing.
\end{credits}
%
%
%
\bibliographystyle{splncs04}
\bibliography{refs}

\newpage



\end{document}


\title{Supplement to: Non-maximizing 
policies that fulfill multi-criterion aspirations in expectation
}
\titlerunning{Supplement: Policies that fulfill multi-criterion aspirations in expectation}
\author{
Simon Dima\inst{1}\orcidID{0009-0003-7815-8238} \and
Simon Fischer\inst{2}\orcidID{0009-0000-7261-3031} \and
Jobst Heitzig\inst{3}\orcidID{0000-0002-0442-8077} \and
Joss Oliver\inst{4}\orcidID{0009-0008-6333-7598}
}
%
%
\institute{
École Normale Supérieure, Paris, France, \email{simon.dima@ens.psl.eu} \and 
Independent Researcher, Cologne, Germany
\and
Potsdam Institute for Climate Impact Research, Potsdam, Germany, \email{heitzig@pik-potsdam.de} \and
Independent Researcher, London, UK, \email{joss.oliver62@gmail.com}
}

\maketitle              

\section{Clipping instead of shrinking}
\subsection{While propagating action-aspirations to state-aspirations}
In section 4.1 of the main text, we define shifted copies $\X_{s'}$
of the action-aspiration $\Easp$,
which would be appropriate candidates for subsequent state-aspi\-rations
if not for the fact that they are not necessarily subsets of $\V^\R(s')$.
In the main text, we resolve this by shrinking $\X_{s'}$ until it fits.

An alternative approach is \emph{clipping}, where we
calculate the sets $\X_{s'}$ and simply set
$\Easp_{s'} = \X_{s'} \cap \V^\R(s')$.

\subsection{While choosing actions and action-aspirations}
In section 4.2.(iii) of the main text,
we shift the aspiration set in a certain direction $y$ and shrink it until it fits into the
action feasibility reference simplices $\Q^\R(s, a_i)$.

Clipping is also an alternative here, as illustrated in figure \ref{figclip}.
To use clipping instead of shrinking, algorithm 2 of the main text
can be modified by replacing the procedure \textsc{ShrinkAspiration}
with the procedure \textsc{ClipAspiration} defined here:
\begin{algorithm*}[h]
\footnotesize
\caption{Clipping variation for algorithm 2 of main text.\label{algCclip}}
\begin{algorithmic}[1]
    \Procedure{ClipAspiration}{$a_i,i$}
        \State $v \gets V^{\pi_i}(s)$ if $i>0$, else $C(\Q^\R(s,a_i))$ \Comment{vertex or center of target} 
        \State $y \gets v - x$ \Comment{shifting direction}
        \State $r \gets \max\{ r\ge 0 : x - ry\in\Easp \}$ \Comment{distance from boundary point $b$ of $\Easp$ to $x$}
        \State $\ell \gets \min\{ \ell\ge 0 : x + \ell y\in\Q^\R(s,a_i) \}$ \Comment{distance from $x$ to reference simplex}
        \label{line:calcl}
        \State $\ell' \gets \max\{ \ell'\ge 0 : x + \ell' y\in\Q^\R(s,a_i) \}$ \Comment{dist. from $x$ to far bdry. of ref. simplex}
        \State $z\gets \min(r+\ell, \ell')y$ \Comment{shifting vector}
        \label{line:calcz}
        \State \Return $\big(z + \Easp_s\big) \cap \Q^\R(s,a_i)$ \Comment{shift and clip (nonempty)}
    \EndProcedure
\end{algorithmic}
\end{algorithm*}

The direction $y$ is determined as in the shrinking variant.
Next, we determine by linear programming the outermost point $b\in\Easp$ lying on the ray from $x$ in the direction $-y$.
$\Easp$ is shifted in direction $y$, stopping either when
the shifted boundary point $b$
(and thus generally a large part of $\Easp$)
enters the reference simplex $\Q^\R(s,a_i)$,
or when the shifted center $x$ exits $\Q^\R(s,a_i)$.
The appropriate shifting distance can be computed using linear programming.
This is done in lines \ref{line:calcl} to \ref{line:calcz} of Algorithm~\ref{algCclip}.
Finally, the shifted copy of $\Easp$ is clipped to $\Q^\R(s,a_i)$ by uniting their $H$-representations to get $\Easp_{a_i}$, i.e., consider the set of all hyperplanes on which a facet of either $\Easp$ or $\Q^\R(s,a_i)$ lies.
Then the vertices of $\Easp_{a_i}$ are computed from the resulting $H$-representation.
The sets $\Easp_{a_i}$ are nonempty because the ray from $x$ in direction $y$ intersects $\Q^\R(s,a_i)$.

\begin{figure}
\centering
\includegraphics[width=.8\textwidth, trim=0 550 700 0, clip]{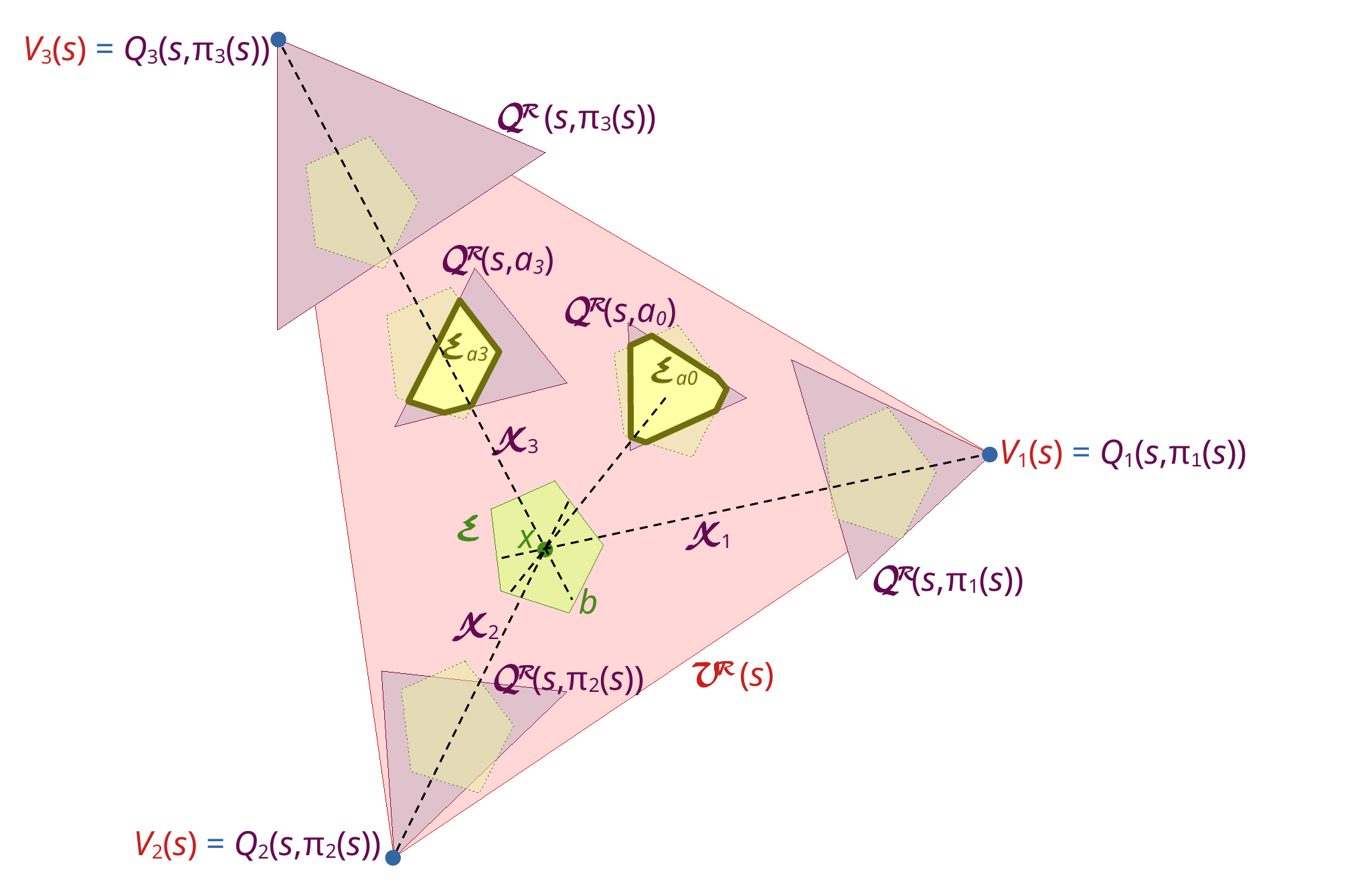}
\caption{Construction of action-aspirations $\Easp_a$ from state-aspiration $\Easp$ and reference simplices $\V^\R(s)$ and $\Q^\R(s,a)$ by shifting and clipping. See main text for details.} \label{figclip}
\end{figure}

\subsection{Advantages and disadvantages of clipping}
Clipping instead of shrinking has the advantage that it produces strictly larger propagated aspiration sets,
which might be desirable as discussed in section 3 of the main text.

However, clipping changes the shape of aspiration sets,
adding up to $d+1$ defining hyperplanes at each step,
and requiring recomputation of the set of vertices which may become combinatorially large.
Therefore, we expect the complexity of the clipping variant to be worse than shrinking,
though we have not studied it formally.

\section{Alternate version of action selection/local policy computation}
Some action selection criteria which we may wish to use are ``farsighted''
and require knowing the future behavior of the policy as well.
Algorithm 2 of the main text does not provide this information,
as it samples candidate actions $a_i$ first before determining the probabilities with which to mix them.
The algorithm \ref{algpi} presented here is an adaptation which does provide the full local policy
before sampling from it to choose the action taken.
It does so by replacing the sampling in lines 4 and 7 of the original algorithm
by a computation of the respective probabilities.
Instead of solving the linear program in line 6 of main-text algorithm 2
for each of the at most $|\A|^{d+2}$ many candidate action combinations $(a_0,\dots,a_{d+1})$,
line \ref{line:singlelp} uses a single linear program invocation manipulating
the average action-aspirations for each direction, $\E_{a_i\sim\pi_i}\Easp_{a_i}$.
This ensures complexity remains linear in $|\A|$.
A solution to this linear program exists since $\E_{a_i\sim\pi_i}\Easp_{a_i} \subseteq \Easp_s + \E_{a_i\sim\pi_i} z_{a_i}$ and $\E_{a_i\sim\pi_i} z_{a_i}$ is a positive multiple of $V^{\pi_i}(s)-x$; the proof of Lemma 2 of the main text is readily adapted.
Finally, the loop in lines \ref{line:loopstart} to \ref{line:loopend}
is necessary as one same action may lie in two distinct directional action sets
$\A_i \neq \A_j$.
 
\begin{algorithm*}[h]
\footnotesize
\caption{Local policy computation (``clip'' and ``shrink'' variants)\label{algpi}}
\begin{algorithmic}[1]
    \Require Reference simplex vertices $V^{\pi_i}(s)$, $Q^{\pi_i}(s,a)$; state $s$; state-aspiration $\Easp_s$.
    \State $x \gets C(\Easp_s)$, $\Easp' \gets \Easp_s - x$ \Comment{center, shape}
    \For{$i = 0, \dots, d+1$}
        \State $\A_i \gets$ DirectionalActionSet$(i)$
        \State $\pi_i \gets$ CandidateActionDistribution$(\A_i)$ \Comment{any probability distribution on $\A_i$}
        \For{$a_i\in\A_i$}
            \State $\Easp_{a_i} \gets$ ClipAspiration$(a_i,i)$ or ShrinkAspiration$(a_i,i)$ \Comment{action-aspiration}
        \EndFor
    \EndFor
    \State solve linear program $p_0=\max$!, $\sum_{i=0}^{d+1} p_i \E_{a_i\sim\pi_i} \Easp_{a_i} \subseteq \Easp_s$ for $p\in\Delta(\{0,\dots,d+1\})$
    \label{line:singlelp}
    \State initialize $\pi\equiv 0$
    \label{line:loopstart}
    \For{$i = 0, \dots, d+1$} \Comment{collect probabilities across directions}
        \For{$a_i\in\A_i$}
            \State $\pi(a_i,\Easp_{a_i}) \gets \pi(a_i,\Easp_{a_i}) + p_i \pi_i(a_i)$
        \EndFor
    \EndFor
    \label{line:loopend}
    \Return $\pi$ \Comment{local policy}
\end{algorithmic}
\end{algorithm*}

\section{Video of an example run}

In the video accessible at the anonymous download link\\
{\small\url{https://mega.nz/file/gUgTzZqB#Co21aSoKQAo7PjeR9_kzxzquHdjlNav0qjp9wKX3NRU},}\\ we illustrate the propagation of aspirations from states to actions in one episode of a simple gridworld environment.

In that environment, the agent starts at the central position (3,3) in a 5-by-5 rectangular grid, can move up, down, left, right, or pass in each of 10 time steps, and gets a Delta $f(s,a,s')=g(s')$ that only depends on its next position $s'$ on the grid. 
The values $g(s')$ are iid 2-d standard normal random values.

In each time-step, the black dashed triangle is the current state's reference simplex $\V^\R(s)$,
the blue triangles are the five candidate actions' reference simplices $\Q^\R(s,a)$,
the red square is the state-aspiration $\Easp$,
the dotted lines connect its center, the red dot, with the vertices of $\V^\R(s)$ or with the centers of the sets $\Q^\R(s,a)$, and the green squares are the resulting action-aspirations $\Easp_a$.

The coordinate system is ``moving along'' with the received Delta, so that aspirations of consecutive steps can be compared more easily. In other words, the received Delta $R_t = \sum_{t'=0}^{t-1}f(s_t,a_t,s_{t+1})$ is added to all vertices so that, e.g., the vertices of the back dashed triangle are shown at positions $R_t+V^{\pi_i}(s_t)$ rather than $V^{\pi_i}(s_t)$.

This run uses the linear volume shrinking schedule $r_{\max}(s)=\big(1-1/T(s)\big)^{1/d}$.
As one can see, the state-aspirations indeed shrink smoothly towards a final point aspiration, which spends the initial amount of ``freedom'' rather evenly across states. This way, the agent manages to avoid drift in this case, so that the eventual point aspiration is still inside the initial aspiration set.

\section{Numerical evidence for reference policy selection complexity}

In order to study the number of trials needed to find the $d+1$ directions that define suitable reference policies for the definition of reference simplicies, we performed numerical experiments with binary-tree-shaped environments of the following type.

Each environment has 10 time steps, each state has two actions, each action can lead to two different successor states, leading to $4^{10}$ many terminal states. Deltas $f(s,a,s')$ are iid uniform variates in $[0,1]^d$. The initial aspiration is $(5,\dots,5)$.

For each $d\in\{1,\dots,8\}$, we perform 1000 independent simulations and find that the average number $k$ of trials until $\Easp_0\in\conv\{v_1,\dots,v_k\}$ is below $2d+1$ for all tested $d$.